\documentclass[]{bytedance_seed}

\usepackage[toc,page,header]{appendix}

\usepackage{minitoc}

\usepackage{amsmath,amssymb,amsthm}
\usepackage{mathtools}
\usepackage{enumitem}
\usepackage{tikz}
\usepackage{tikz}
\usetikzlibrary{arrows.meta,positioning,calc,decorations.pathreplacing,decorations.pathmorphing}
\usepackage{booktabs}

\theoremstyle{plain}
\newtheorem{proposition}{Proposition}

\newtheorem{corollary}{Corollary}

\theoremstyle{definition}
\newtheorem{definition}{Definition}
\newtheorem{assumption}{Assumption}

\theoremstyle{remark}
\newtheorem*{remark}{Remark}

\title{Who Prices Cognitive Labor in the Age of Agents? \\
Compute-Anchored Wages}

\author{Siqi Zhu}

\affiliation{University of Illinois Urbana-Champaign}

\abstract{
A natural intuition about the economics of AI agents is that, because agents can be replicated at near-zero marginal cost, they constitute a labor input in infinitely elastic supply, and therefore drive cognitive-labor wages to zero. We argue this framing is wrong in mechanism but partially correct in conclusion, and that the correction matters for both theory and policy. \textbf{Agents are not labor; they are a production technology that converts compute capital $K_c$ into effective units of cognitive labor $L_A$.} Once this is recognized, the elastic-supply margin that anchors the equilibrium wage migrates from the labor market to the compute capital market. Building on the classic factor-pricing framework \citep{mankiw2020}, we derive a \emph{Compute-Anchored Wage} (CAW) bound stating that, on tasks where human and agent cognitive labor are substitutes, the competitive human wage is bounded above by $\lambda \cdot k \cdot r_c$, where $r_c$ is the rental rate of compute capital, $k$ is the compute intensity of one effective agent-labor unit, and $\lambda$ is the relative human-to-agent productivity. We generalize the result through constant elasticity of substitution (CES) aggregation, separate substitutable from complementary tasks, and discuss factor-share consequences. The conclusion is concise: \emph{the price-setter for cognitive labor is no longer the labor market.}
}

\begin{document}
\maketitle


\section{Introduction}

The standard textbook account of wage determination, as presented by \citet{mankiw2020}, has two ingredients: a downward-sloping labor demand curve given by the marginal product of labor, and a labor supply curve determined by household time allocation and demographics. Equilibrium wages clear the labor market.

The arrival of capable AI agents disturbs this picture, and the analytical question is where the disturbance enters. A tempting accounting models agents as a new labor input that substitutes for human cognitive labor, is reproducible at near-zero marginal cost, and has a supply curve horizontal at zero, and then reads off a collapsing wage from that horizontal supply. We argue this accounting misplaces the elastic margin. Agents are not a labor input; they are a technology that converts compute capital into effective cognitive labor. Their supply elasticity is therefore inherited from the supply elasticity of compute capital, which is finite and governed by physical and political-economic constraints such as semiconductor fab capacity, electricity, water, land, and geopolitics. The correct model reroutes the price-determination question through the compute capital market rather than the labor market.

\paragraph{A concrete instance.} Consider a junior contract-review paralegal whose work consists largely of clause extraction, redlining against templates, and summary memos. A frontier large language model performs each of these tasks at quality close to or above the paralegal's, at a marginal compute cost on the order of single-digit dollars per labor-hour-equivalent at 2024--2025 prices (Section~\ref{sec:calib}). The competitive wage on this paralegal's substitutable hours is therefore not set by the supply of paralegals; it is set by the rental rate of compute capital, scaled by an algorithmic constant. The same logic applies to first-pass equity-research drafting, customer-support triage, and a long list of cognitive tasks that the popular discussion classifies as ``automatable'' without identifying the specific market in which the new price is determined.

This paper makes that rerouting explicit, derives a closed-form ceiling on competitive wages that we call the Compute-Anchored Wage (henceforth CAW), generalizes the substitution structure via a constant elasticity of substitution (CES) aggregator, separates wage effects across heterogeneous tasks, calibrates the bound to current compute prices, and discusses limitations and policy. Our claim is that the analytical primitive of where to put the elastic supply has been miscoded in much current discussion, and that fixing this miscoding yields a sharp, testable prediction about cognitive-labor pricing. The mathematical content of the bound is standard cost minimization; the substantive content is the identification of the elastic margin.

The remainder of the paper is organized as follows. Section~2 reviews the related literature and locates our contribution relative to the task-based automation framework of \citet{acemoglu2018,acemoglu2022} and the capital-skill complementarity tradition of \citet{korv2000}. Section~3 sets up the textbook factor-pricing model, Section~4 reformulates AI agents as a capital-to-labor conversion technology, and Section~5 derives the CAW bound. Section~6 generalizes the bound to imperfect substitution via CES, while Section~7 calibrates it to current compute prices. Section~8 visualizes the migration of the price-setter, Section~9 develops the cross-task heterogeneity prediction, and Section~10 discusses macro factor-share consequences and policy levers. Section~\ref{sec:limits} states limitations.

\section{Related Work}

Our work intersects five strands of literature. We summarize each, and then state explicitly what the Compute-Anchored Wage (CAW) framework adds.

\paragraph{Factor pricing and capital--skill complementarity.} The marginal-productivity theory of factor pricing in competitive markets, codified in \citet{mankiw2020}, supplies the entire formal apparatus we use; we make no modifications to the underlying theory, and the only re-coding is in how an AI agent enters the production function. Closer to our setup, \citet{korv2000} show that a CES production function in which capital equipment complements skilled labor and substitutes for unskilled labor matches the joint behavior of the skill premium and capital--output ratios over decades. We treat compute capital $K_c$ as a factor that, when paired with the inference stack $\phi$, becomes a quasi-substitute for human cognitive labor on a measurable subset of tasks. The crucial difference is that in \citet{korv2000} the substitution margin is between physical equipment and \emph{unskilled} labor; in CAW the margin is between compute capital and \emph{cognitive} labor previously regarded as the complement of capital. CAW thus inherits \citet{korv2000}'s machinery but inverts the sign of the implied skill premium on the substitutable margin.

\paragraph{Task-based automation (Acemoglu--Restrepo).} The closest theoretical antecedent is the task-based automation framework of Acemoglu and Restrepo (henceforth A--R) \citep{acemoglu2018,acemoglu2019,acemoglu2020,acemoglu2022}, in which capital automates a contiguous subset of tasks, displacing labor on the automated margin and reinstating labor through the creation of new tasks. Our framework can be read as a specialization of A--R to the case where the automating capital is compute and the displaced labor is cognitive; the price effect we isolate ($W_H \leq \lambda k r_c$) is the explicit equilibrium consequence of A--R's displacement effect when the automating capital is in elastic but finite supply. We extend A--R in three ways. First, we identify a specific elastic margin, namely the compute capital market, that anchors the equilibrium wage on automated tasks and traces the price-determination problem to a measurable rental rate $r_c$. Second, we make the task partition explicit through an elasticity-of-substitution parameter $\sigma$ that is in principle estimable from observed factor demands. Third, we discuss the political economy of compute-market concentration as a determinant of $r_c$ that has no analogue in standard A--R. The reinstatement effect of A--R is consistent with our complementary-task subset $T_C$ and is preserved.

\paragraph{Skill-biased technical change and occupational exposure to AI.} \citet{katz1992} document the rising college premium and frame technology--skill complementarity as the principal driver. \citet{autor2003} sharpen this into a task-based account in which information technology substitutes for routine cognitive and manual tasks while complementing non-routine analytic and interpersonal tasks, and subsequent empirical work (\citealp{autor2013}; \citealp{autor2015}; \citealp{goldinkatz2008}) traces job polarization and the long-run race between education and technology. A nascent empirical literature now measures occupation-level exposure to large language models: \citet{eloundou2023} estimate that 80\% of US workers could see at least 10\% of their tasks affected, \citet{felten2023} provide an alternative occupational exposure score, \citet{brynjolfsson2023} document a 14\% productivity gain among customer-support agents using a generative-AI assistant, and \citet{noyzhang2023} find a roughly 40\% time reduction and 18\% quality improvement on professional writing tasks. Our framework predicts a directional inversion \emph{within} cognitive labor itself: the salient axis becomes the substitutable--complementary mix on which an occupation is exposed to AI agents, rather than its position on a one-dimensional skill ladder, and these exposure studies provide the natural empirical input to the $T_S$/$T_C$ partition and to estimating $\sigma$ on each task.

\paragraph{AI in macroeconomics, general-purpose technologies, and compute supply.} \citet{aghion2017} model AI as a sequence of new general-purpose technologies that automate task production and study balanced-growth implications; \citet{korinek2019} examine distributional consequences of AI under capital--labor substitution; \citet{trammellkorinek2023} survey the macroeconomics of transformative AI; and \citet{korinek2023} discusses large language models specifically. \citet{bresnahan1995} formalize GPTs as innovations whose value derives from co-invention in downstream sectors, and \citet{goldfarb2023} provide empirical support for treating machine learning as a GPT. On the supply side, \citet{sevilla2022} document the doubling time of frontier-model training compute, and \citet{cottier2024} estimate the rising cost of frontier-model training, pinning down the empirical content of the $r_c$ time path that drives CAW. CAW operates at a different level of abstraction from the macro and GPT literatures: rather than modeling growth dynamics or diffusion, it identifies an equilibrium pricing relation that any of these dynamic models must satisfy on the substitutable margin in any period. We do not model compute supply explicitly but rely on the stylized fact that it is finite and only moderately elastic in the medium run because of fab capacity, energy, water, land, and policy.

\paragraph{Declining labor share.} \citet{karabarbounis2014}, \citet{piketty2014}, and \citet{autorvanreenen2020} document the long-run decline in the labor share and the rise of superstar firms, and \citet{susskind2020} provides a non-technical synthesis. CAW refines this discussion by identifying a specific channel, the compute share of capital income, through which capital-income concentration is now operating, and by predicting that the same mechanism will compress wages within cognitive labor.

\paragraph{Summary of contribution.} Our central analytical re-coding is that agents are a capital-to-labor conversion technology rather than a labor input, and we trace its equilibrium consequences. Relative to the closest antecedent in A--R, we add three things. We identify the compute capital market as the elastic margin that prices substitutable cognitive labor; we propose $\sigma$ on each task as the empirical primitive that replaces the binary ``automated or not'' coding; and we connect the wage bound to a quantitatively tractable factor price $r_c$ for which there is now an active spot and contract market. The remainder of the paper develops the consequences of that re-coding.

\section{Setup: Factor Markets in the Mankiw Framework}

Following the textbook factor-market model, consider a representative competitive firm with a constant-returns-to-scale production technology
\begin{equation}
Y = F(K, L), \label{eq:standard}
\end{equation}
where $K$ is physical capital and $L$ is labor. Profit maximization in a competitive output market with price $P$ yields
\begin{align}
W = P \cdot \frac{\partial F}{\partial L}, \quad
r = P \cdot \frac{\partial F}{\partial K}.
\end{align}
Factor prices equal the value of the marginal product. Equilibrium in the labor market is
\begin{equation}
L^d(W) = L^s(W),
\end{equation}
with $L^s$ determined by household time allocation and demographic factors. The wage $W$ is set by where these curves intersect.

The question is what happens when an AI agent becomes available as a partial substitute for $L$. The temptation is to add a new labor type $L_A$ with infinitely elastic supply at zero price, mechanically forcing $W \to 0$ on the substitutable margin. We now argue this addition is the wrong primitive.

\section{Reformulation: Agents as Capital-to-Labor Conversion}

\begin{definition}[Agent-Produced Cognitive Labor]
An \emph{agent-produced cognitive labor unit} $L_A$ is the cognitive-labor-equivalent output produced when a fixed bundle of compute capital $k$ (GPU-hours, energy, memory, bandwidth, model weights amortized as IP rents) is operated for one unit of time. Formally, $L_A = \phi(K_c)$, where $\phi$ is increasing and at the relevant scale approximately linear, $\phi(K_c) \approx K_c / k$.
\end{definition}

\begin{remark}
The function $\phi$ embeds the model architecture, training run, and inference stack. Improvements in algorithmic efficiency (distillation, speculative decoding, mixture-of-experts routing, KV-cache reuse) raise $\phi$ for given $K_c$ and thus reduce $k$. Crucially, $\phi$ is a \emph{technology}, not a behavioral primitive of households; it does not enter any labor supply problem.
\end{remark}

\begin{remark}[Heterogeneity within $K_c$]
The compound input $K_c$ pools three economically distinct objects. The first is physical compute, including GPUs, accelerators, data-center capacity, and energy, which is rival and finitely supplied. The second is model-weight intellectual property, which is non-rival and reproducible at zero marginal cost, so that its rents are pinned down by training sunk costs and licensing structure. The third is sunk training capital that is amortized into per-token inference prices. The CAW bound below is governed primarily by the variable inference component, with the IP-rent component entering as a markup over marginal compute cost. We discuss this further in Section~\ref{sec:limits}.
\end{remark}

The augmented production function is
\begin{equation}
Y = F(K_o, L_H, L_A) = F\bigl(K_o, L_H, \phi(K_c)\bigr), \label{eq:augmented}
\end{equation}
where $K_o$ is non-compute capital, $L_H$ is human cognitive labor, and $K_c$ is compute capital with rental rate $r_c$ determined in the compute capital market. The firm's first-order conditions become
\begin{align}
W_H &= P \cdot \frac{\partial F}{\partial L_H}, \\
r_c &= P \cdot \frac{\partial F}{\partial L_A} \cdot \phi'(K_c).
\end{align}
The crucial observation: $L_A$ has no household supply curve. Its supply derives entirely from the supply of $K_c$, which is governed by fab capacity, energy infrastructure, data-center construction lead times, and policy. These are finite and relatively inelastic in the short run, only moderately elastic in the long run.

\section{The Compute-Anchored Wage Bound}

We state the central proposition first under the strong, illustrative case of perfect substitution, then generalize.

\begin{assumption}[Perfect substitutability on the margin]\label{assn:perfect}
There exists a set of tasks on which one unit of human cognitive labor and $\lambda$ units of agent-produced labor are perfect substitutes in $F$, so that effective cognitive labor on these tasks aggregates as $L_{\text{eff}} = L_H + \lambda^{-1} L_A$. Equivalently, one human unit produces the same effective cognitive output as $\lambda$ agent units, so that $\lambda > 1$ corresponds to humans being more productive per unit, and $\lambda < 1$ to agents dominating.
\end{assumption}

\begin{proposition}[Compute-Anchored Wage]\label{prop:caw}
Under Assumption~\ref{assn:perfect} and competitive factor markets, the equilibrium human cognitive wage on the substitutable task set satisfies
\begin{equation}
\boxed{\;W_H \;\leq\; \lambda \cdot k \cdot r_c\;} \label{eq:caw}
\end{equation}
in any equilibrium with $L_H > 0$. Moreover, in any equilibrium with both $L_H > 0$ and $L_A > 0$, the bound holds with equality, $W_H = \lambda k r_c$.
\end{proposition}

\begin{proof}[Proof sketch]
The unit cost of effective cognitive labor on the substitutable task is $\min\{W_H,\, \lambda k r_c\}$, since $\lambda$ units of $L_A$ substitute for one unit of $L_H$ and each unit of $L_A$ requires $k$ units of compute at rental rate $r_c$. Three cases follow. When $W_H < \lambda k r_c$, firms substitute fully toward $L_H$ on this task, so that $L_A^d = 0$ on $T_S$ and the bound is slack but nonbinding because no agent-produced labor is used. When $W_H > \lambda k r_c$, firms substitute fully toward $L_A$, so that $L_H^d = 0$ on $T_S$; any positive supply of $L_H$ at this wage on this task is unemployed, and the wage cannot be sustained in any equilibrium with positive employment of $L_H$ on $T_S$. Hence $W_H \leq \lambda k r_c$ in any such equilibrium. Finally, interior coexistence with $L_H, L_A > 0$ on $T_S$ requires the marginal indifference $W_H = \lambda k r_c$. The detailed cost-minimization derivation is in Appendix~B.
\end{proof}

\begin{corollary}[Migration of the price-setter]
The equilibrium wage on substitutable cognitive tasks is determined by the parameters of the compute capital market $(k, r_c)$ and the technology parameter $\lambda$. The labor supply curve $L_H^s$ does \emph{not} appear in the binding condition. The price-setting margin has migrated from the labor market to the compute capital market.
\end{corollary}

This is the formal content of our claim. Three clarifications are in order: wages do not collapse to zero but to $\lambda k r_c$, which can be high or low depending on compute-market conditions; human workers need not become unemployed, since they may relocate to complementary tasks (Section~\ref{sec:tasks}); and the bound applies only on tasks where Assumption~\ref{assn:perfect} approximately holds, not in all sectors.

\section{CES Generalization: Imperfect Substitution}

Perfect substitution overstates the case. Generalize via constant elasticity of substitution. Let
\begin{equation}
L_{\text{eff}} = A\bigl[\alpha L_H^{\rho} + \beta L_A^{\rho}\bigr]^{1/\rho}, \qquad \sigma = \frac{1}{1-\rho}, \label{eq:ces}
\end{equation}
with $\sigma \in (0, \infty)$ the elasticity of substitution between human and agent cognitive labor, and CES weights $\alpha, \beta > 0$. The effective unit cost of an agent-produced labor is its compute cost, $W_A^{\text{eff}} \equiv k r_c$, since $L_A = K_c/k$ and the rental rate of $K_c$ is $r_c$. Cost minimization of \eqref{eq:ces} for one unit of $L_{\text{eff}}$ delivers the conditional factor demands and the relative-wage condition
\begin{equation}
\frac{W_H}{W_A^{\text{eff}}} = \frac{\alpha}{\beta}\left(\frac{L_H}{L_A}\right)^{\rho - 1} = \frac{\alpha}{\beta}\left(\frac{L_H}{L_A}\right)^{-1/\sigma}. \label{eq:relwage}
\end{equation}
With normalization $\alpha/\beta = \lambda$ on the perfect-substitute limit, \eqref{eq:relwage} reduces to $W_H = \lambda W_A^{\text{eff}}$ as $\sigma \to \infty$, recovering the CAW bound \eqref{eq:caw}.

\begin{proposition}[Compute-driven wage compression]\label{prop:compression}
Hold the human-labor supply $L_H^s$ and the level of demand for $L_{\text{eff}}$ fixed. Consider an exogenous increase in compute capital supply $\bar K_c$ (equivalently, a fall in $k$ via $\phi$ improvement) at a fixed rental rate $r_c$ that lowers $W_A^{\text{eff}}$. Then in the new equilibrium the human cognitive wage falls with semi-elasticity
\begin{equation}
\frac{\partial \log W_H}{\partial \log W_A^{\text{eff}}} \;=\; 1 - \frac{1}{\sigma}\cdot\frac{\partial \log(L_H/L_A)}{\partial \log W_A^{\text{eff}}},
\end{equation}
and in the polar limit of perfectly elastic $L_H^s$ this collapses to $\partial \log W_H / \partial \log W_A^{\text{eff}} = 1$. As $\sigma \to \infty$, the CES bound collapses to the perfect-substitute CAW bound \eqref{eq:caw}; as $\sigma \to 0$ (Leontief), $L_H$ and $L_A$ are used in fixed proportion, the binding factor is whichever is shorter in supply, and the comparative-static channel from $W_A^{\text{eff}}$ to $W_H$ vanishes.
\end{proposition}

The CES form makes precise an empirically tractable claim: \emph{the magnitude of CAW pressure on a given task is governed by the elasticity of substitution between human and agent cognitive labor on that task}. This is the right object for empirical estimation, replacing the binary ``AI replaces / does not replace'' framing common in policy discourse. \citet{eloundou2023}, \citet{felten2023}, \citet{brynjolfsson2023}, and \citet{noyzhang2023} provide the natural empirical input.

\section{A Numerical Calibration of CAW}\label{sec:calib}

To give the bound \eqref{eq:caw} empirical traction, we now plug in plausible 2024--2025 numbers. The exercise is illustrative, not estimation; the goal is to show that $\lambda k r_c$ takes economically meaningful values that vary by orders of magnitude across tasks.

\paragraph{Compute rental rate $r_c$.} On-demand H100 GPU rental from major cloud providers traded in the $\$2$--$\$5$/GPU-hour range in 2024, with multi-year contract pricing closer to $\$1.50$/GPU-hour. We take $r_c = \$2$/GPU-hour as a midpoint. Frontier-model inference is typically priced per million tokens; converting to GPU-hour-equivalents using public throughput benchmarks for a 70B-class model on an H100 yields roughly the same order of magnitude.

\paragraph{Compute intensity $k$.} For a frontier reasoning agent producing sustained, high-quality output, current inference stacks consume on the order of $0.5$--$2$ H100-hours of compute to deliver one ``hour'' of effective senior-knowledge-worker output, depending on whether the workload is interactive (heavy KV-cache reuse) or batched. We take $k = 1$ H100-hour per agent-labor-hour for a frontier model and $k = 0.05$ H100-hour per agent-labor-hour for a small distilled model on a substitutable subtask.

\paragraph{Productivity ratio $\lambda$.} The empirical literature on LLM productivity gains \citep{brynjolfsson2023, noyzhang2023} reports time savings of 14--40\% on substitutable tasks, with quality at or above human baseline. We take $\lambda \in \{0.5, 1.0, 2.0\}$ to span the cases where agents are absolutely more productive ($\lambda < 1$), at parity ($\lambda = 1$), and where humans retain a productivity edge ($\lambda > 1$).

\paragraph{Implied CAW.} Combining these, the bound $W_H \leq \lambda k r_c$ implies the per-hour ceilings shown in Table~\ref{tab:calib}.

\begin{table}[h]
\centering
\small
\begin{tabular}{lccc}
\toprule
& Distilled small model & Frontier model & Frontier model \\
& ($k=0.05$, $r_c=\$2$) & ($k=1$, $r_c=\$2$) & ($k=1$, $r_c=\$5$) \\
\midrule
$\lambda = 0.5$ (agent-favored) & \$0.05/h & \$1.00/h & \$2.50/h \\
$\lambda = 1.0$ (parity)        & \$0.10/h & \$2.00/h & \$5.00/h \\
$\lambda = 2.0$ (human-favored) & \$0.20/h & \$4.00/h & \$10.00/h \\
\bottomrule
\end{tabular}
\caption{Illustrative CAW ceiling $\lambda k r_c$ in US\$/hour on substitutable cognitive tasks, under 2024--2025 compute prices. Read each cell as: \emph{the binding human wage on tasks where this $(\lambda, k, r_c)$ configuration applies}.}
\label{tab:calib}
\end{table}

Two implications are immediate. First, on tasks where small distilled models suffice (high-volume classification, summarization, first-pass document review), CAW is already binding well below any plausible human reservation wage; the wage on such tasks is effectively pinned at the marginal-product floor. Second, on tasks requiring frontier-model reasoning, CAW currently sits between roughly \$1 and \$10/hour. Any human cognitive labor on tasks where Assumption~\ref{assn:perfect} approximately holds and the frontier model is competent must be priced under that ceiling to retain employment. As $\phi$ improves and $k$ falls, every cell of Table~\ref{tab:calib} moves down monotonically.

\paragraph{Sensitivity.} The ceiling is linear in each of $\lambda$, $k$, and $r_c$. A doubling of compute prices (e.g., from a supply shock or geopolitical disruption) doubles CAW; a halving of $k$ from algorithmic improvement halves it. Thus the empirical content of the framework is the joint trajectory $(k_t, r_{c,t})$ on a task-by-task basis, with $\sigma$ governing how rapidly the relevant occupational wage tracks that trajectory.

\section{Visualizing the Migration of the Price-Setter}

Figure~\ref{fig:migration} traces the migration of the price-setter across three panels: the textbook labor market in (a), the compute capital market in (b), and the CAW-anchored cognitive labor market on $T_S$ in (c). The reading order is important. Agent labor has \emph{no} household supply curve; its supply is derived from the supply of compute capital $K_c$, which in the short run is steep (panel (b)) due to fab capacity, energy, and data-center lead times. Compute demand therefore pins the rental rate $r_c^{*}$ in (b). Conditional on that $r_c^{*}$, the horizontal line at $\bar W_H = \lambda k r_c^{*}$ in panel (c) is a wage \emph{ceiling} on human cognitive labor on $T_S$, not an agent-labor supply curve: it does not assert that agent supply is infinitely elastic. The human-labor supply curve $L_H^s$ is drawn but does not determine the equilibrium wage on $T_S$. In general equilibrium, shifts in compute demand move $r_c^{*}$ in (b), and the CAW line in (c) shifts in lockstep.

\begin{figure}[h]
\centering
{\definecolor{uiucorange}{HTML}{E84A27}%
\begin{tikzpicture}[scale=0.7, every node/.style={font=\small}]

\begin{scope}[xshift=0cm]
\draw[->,thick] (0,0) -- (5.5,0) node[right] {$L$};
\draw[->,thick] (0,0) -- (0,4.5) node[above] {$W$};
\draw[uiucblue,very thick] (0.3,4) -- (5,0.5) node[right,black] {$L^d$};
\draw[uiucblue,very thick] (0.5,0.3) -- (5,4) node[right,black] {$L^s$};
\draw[dashed] (0,2.15) node[left] {$W^*$} -- (2.71,2.15);
\draw[dashed] (2.71,0) node[below] {$L^*$} -- (2.71,2.15);
\fill[uiucblue] (2.71,2.15) circle (2pt);
\node[above=0.3cm] at (2.75,4.6) {\textbf{(a) Classic labor market}};
\node[align=center,below=0.1cm] at (2.75,-0.6) {Wage set by\\ $L^d \cap L^s$};
\end{scope}

\begin{scope}[xshift=7.5cm]
\draw[->,thick] (0,0) -- (5.5,0) node[right] {$K_c$};
\draw[->,thick] (0,0) -- (0,4.5) node[above] {$r_c$};
\draw[red,very thick] (3.0,0.3) -- (3.4,4.2) node[right,black] {$K_c^{s}$ (short-run)};
\draw[uiucblue,very thick] (0.3,4) -- (5,0.5) node[right,black] {$K_c^{d}$};
\draw[dashed] (0,1.85) node[left] {$r_c^{*}$} -- (3.2,1.85);
\draw[dashed] (3.2,0) node[below] {$\bar K_c$} -- (3.2,1.85);
\fill[red] (3.2,1.85) circle (2pt);
\node[above=0.3cm] at (2.75,4.6) {\textbf{(b) Compute capital market}};
\node[align=center,below=0.1cm] at (2.75,-0.6) {Inelastic short-run supply\\ pins $r_c^{*}$};
\end{scope}

\begin{scope}[xshift=15cm]
\draw[->,thick] (0,0) -- (5.5,0) node[right] {$L_H$};
\draw[->,thick] (0,0) -- (0,4.5) node[above] {$W_H$};
\draw[uiucblue,very thick] (0.3,4) -- (5,0.5) node[right,black] {$L_H^d$};
\draw[uiucblue,very thick,dashed] (0.5,0.3) -- (5,4) node[right,black] {$L_H^s$};
\draw[red,very thick] (0,1.85) -- (5,1.85) node[right,red] {CAW: $\bar W_H = \lambda k r_c^{*}$};
\draw[dashed] (0,1.85) -- (3.19,1.85);
\draw[dashed] (3.19,0) node[below] {$L_H^d(\bar W_H)$} -- (3.19,1.85);
\fill[red] (3.19,1.85) circle (2pt);
\node[above=0.3cm] at (2.75,4.6) {\textbf{(c) Cognitive labor on $T_S$}};
\node[align=center,below=0.1cm] at (2.75,-0.6) {Wage \emph{ceiling} given $r_c^{*}$;\\ not an agent-labor supply curve};
\end{scope}

\draw[->,very thick,red,decorate,decoration={snake,amplitude=1pt,segment length=6pt}]
  (11.2,1.85) -- (14.5,1.85);
\node[red,align=center,font=\scriptsize] at (12.85,2.3) {$r_c^{*}\!\mapsto\!\lambda k r_c^{*}$};

\draw[->,very thick,red,decorate,decoration={snake,amplitude=1pt,segment length=8pt}]
  (5.6,5.9) to[bend left=10] (15.1,5.9);
\node[align=center,red,font=\small] at (10.4,6.8) {migration of the price-setter};

\end{tikzpicture}}
\caption{Migration of the price-setting margin. \textbf{(a)} In the classic framework the wage is determined by $L^d \cap L^s$. \textbf{(b)} The compute capital market: short-run supply $K_c^{s}$ is steep (inelastic) due to fab capacity, energy, and data-center lead times, so compute demand $K_c^{d}$ pins the equilibrium rental rate $r_c^{*}$. \textbf{(c)} Cognitive labor on substitutable tasks $T_S$, drawn \emph{conditional on the equilibrium $r_c^{*}$ inherited from panel~(b)}: the human-labor supply curve $L_H^s$ is no longer the binding constraint (drawn dashed); the wage is capped by the horizontal line $\bar W_H = \lambda k r_c^{*}$, and employment is $L_H^d(\bar W_H)$. The horizontal CAW line is \emph{not} an agent-labor supply curve, and the framework does not assume that agent labor is in infinitely elastic supply. It is the wage \emph{ceiling} conditional on $r_c^{*}$. In general equilibrium, any shift in compute demand re-prices $r_c^{*}$ in panel~(b); the CAW line in panel~(c) then shifts in step.}
\label{fig:migration}
\end{figure}
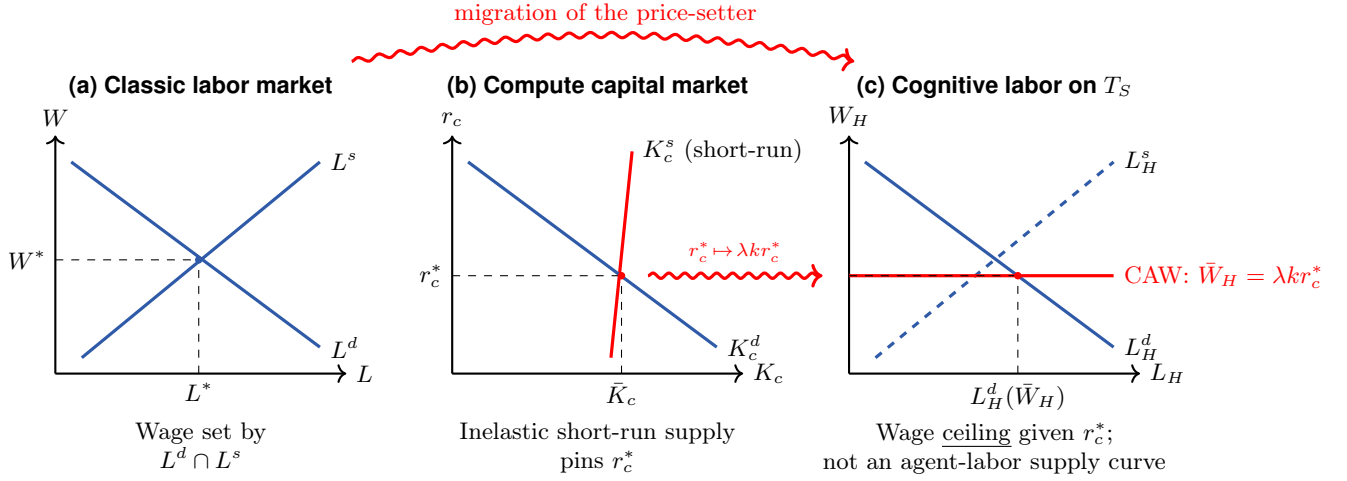

\section{Task Heterogeneity: A Directional Inversion of Skill-Biased Technical Change (SBTC)}\label{sec:tasks}

\citet{katz1992} and the subsequent skill-biased technical change (SBTC) literature document that information technology has historically complemented high-skill cognitive labor and substituted for routine labor \citep{autor2003,autor2013,autor2015}. The CAW framework predicts a directional inversion \emph{within} cognitive labor.

We partition tasks into two sets. The first is the set of \emph{substitutable cognitive tasks} $T_S$, comprising drafting, code generation against specifications, summarization, first-pass analysis, scheduling, retrieval, and classification. On $T_S$, $\sigma$ is large, CAW binds, and $W_H$ is anchored by $\lambda k r_c$; the empirical exposure scores of \citet{eloundou2023} and \citet{felten2023} provide useful proxies for membership in $T_S$. The second is the set of \emph{complementary cognitive tasks} $T_C$, comprising judgment under deep uncertainty, accountability and legal liability, relational and political work, cross-domain integration, taste, principal-agent monitoring, and any task where $\partial F / \partial L_H$ is increasing in $L_A$. On $T_C$, $\sigma$ is small or the cross-partial $\partial^2 F / \partial L_H \partial L_A > 0$, so that $W_H$ \emph{rises} with $L_A$.

The wage distribution across cognitive workers therefore widens, but the relevant axis is no longer traditional skill; it is the $T_S$/$T_C$ exposure mix of each occupation. To make this concrete, consider two occupations with similar formal credentials. A junior contract-review paralegal performs work that is roughly 80\% on $T_S$ (clause extraction, redlining against templates, summary memos) and 20\% on $T_C$ (escalation judgment). A senior litigation associate, by contrast, performs work that is roughly 30\% on $T_S$ (document review, brief drafting) and 70\% on $T_C$ (case strategy, client management, courtroom work). Under CAW the paralegal's wage is dominated by $\lambda k r_c$ on the substitutable component and is squeezed downward as $k$ falls, whereas the associate's wage is dominated by complementary tasks and may rise. This is consistent with the early empirical findings of \citet{brynjolfsson2023} that productivity gains are concentrated in less-experienced workers, but it does not by itself imply that less-experienced workers gain in compensation, since the same tasks that they used to monopolize have been priced down. Two occupations with identical traditional skill requirements but different $T_S$/$T_C$ shares will diverge sharply in compensation. This is a testable cross-sectional prediction distinct from canonical SBTC.

\section{Macro Implications: Factor Shares}

Aggregating across tasks, define the labor share $s_L = W_H L_H / Y$. As compute substitutes for human cognitive labor on $T_S$, the wage bill $W_H L_H$ shrinks on those tasks, while the compute rental bill $r_c K_c$ grows. The capital share rises. Recipients of the rising capital share are owners of compute infrastructure, energy producers, and holders of model intellectual property; these need not coincide with the historical owners of physical capital.

This connects the CAW framework to the literature on declining labor shares \citep{karabarbounis2014,autorvanreenen2020} and the long-run dynamics emphasized by \citet{piketty2014}, with an important refinement: under CAW, the capital share rises specifically through the \emph{compute} channel. Policy interventions targeting that channel (compute taxation, public compute provision, antitrust on accelerator markets, energy policy) have first-order effects on the cognitive-labor wage distribution that interventions targeting the labor market do not. We discuss four such levers in turn.

\paragraph{Compute taxation.} A tax on $r_c$ raises the CAW ceiling proportionally. Incidence depends on the elasticity of compute supply and on the elasticity of demand for substitutable cognitive output. With moderately inelastic compute supply in the short run, much of the tax falls on compute owners; in the long run, with more elastic capacity expansion, incidence shifts toward output prices and ultimately toward human wages on $T_S$ via the bound. Pigouvian arguments based on energy externalities and Ramsey arguments based on the relative inelasticity of $K_c$ supply both push toward positive optimal $\tau_c > 0$.

\paragraph{Public compute provision.} A public option that supplies compute at marginal cost compresses the markup component of $r_c$ and tightens the CAW ceiling. The effect is symmetric to compute taxation in sign on $r_c$ but distributionally different: public provision lowers $W_H$ on $T_S$, raising the consumer surplus of cognitive output buyers without redistributing to workers on $T_S$.

\paragraph{Antitrust on accelerator markets.} If the accelerator market is concentrated, $r_c$ contains a markup over marginal cost. Antitrust enforcement that erodes that markup again lowers the CAW ceiling. The substantive question is whether the resulting consumer-surplus gains in cognitive output exceed the wage compression on $T_S$; under standard welfare assumptions they do, but the distributional consequences are large.

\paragraph{Energy policy.} A binding fraction of $r_c$ is electricity cost. Policy that lowers the levelized cost of electricity to data centers (transmission build-out, nuclear permitting, renewable subsidies) operates through $r_c$ in the same direction as capacity expansion. Because CAW is a price relationship that propagates through tradable cognitive output, its incidence is largely national to global rather than local, even when underlying electricity costs vary regionally.

\section{Limitations and Boundary Conditions}\label{sec:limits}

Several caveats would invalidate or complicate the argument and deserve discussion.

\paragraph{Jevons effects.} A fall in the unit cost of cognitive labor may expand demand for cognitive output enough to raise total $L_H$ employment even on $T_S$ tasks. CAW bounds the \emph{wage}, not the wage \emph{bill}, so whether total compensation rises or falls depends on the demand elasticity for cognitive output.

\paragraph{Ricardian comparative advantage.} Even if agents are absolutely more productive on every task, humans retain employment via comparative advantage. Comparative advantage, however, determines \emph{allocation} rather than \emph{price}, so that wages on the substitutable margin remain anchored.

\paragraph{Non-productivity wage components.} Liability, accountability, signaling, trust, and physical co-presence add a non-marginal-product premium to human labor that the production-function setup does not capture. CAW bounds the marginal-product component of wages, not the total.

\paragraph{Compute-market structure.} The derivation assumes competitive compute markets. If compute is monopolized, vertically integrated with model providers, or politically rationed, then $r_c$ is no longer a competitive rental rate and the bound is replaced by a markup-adjusted version. The argument then becomes a claim about \emph{political economy} as much as about prices.

\paragraph{Endogenous $\phi$.} Algorithmic improvement reduces $k$ over time, so that CAW is a moving target, declining secularly even at constant $r_c$. The relevant long-run object is the joint trajectory of $(k_t, r_{c,t})$. As a first-pass dynamic statement, holding $r_c$ constant and assuming exponential improvement in $\phi$ at rate $g$, the implied CAW trajectory is $\bar W_H(t) = \lambda k_0 e^{-g t} r_c$, which converges to zero as $t \to \infty$ unless arrested by hardware bottlenecks or model-quality saturation.

\paragraph{Endogenous task boundaries.} The partition $T_S \cup T_C$ shifts with the capability frontier, with tasks migrating from $T_C$ to $T_S$ as agents improve. The empirical content of the framework therefore depends on a measurable, time-indexed task taxonomy.

\paragraph{Heterogeneity of $K_c$.} As noted in the remark on $K_c$ heterogeneity in the reformulation, $K_c$ is a composite of physical compute (rival, competitive), model-weight IP (non-rival, often monopolistic), and sunk training capital. A more refined version of CAW would carry these as separate factors with their own pricing equations.

\section{Conclusion}

Our claim can be stated in one sentence: \emph{on tasks where AI agents substitute for human cognitive labor, the equilibrium wage ceiling is set in the compute capital market, not the labor market.} The standard textbook framework already contains all the machinery needed to see this; the only required correction is to recognize agents as a capital-to-labor conversion technology rather than a labor input. Once that correction is made, the Compute-Anchored Wage bound $W_H \leq \lambda k r_c$ follows directly from competitive cost minimization, and a number of empirical and policy questions reorient accordingly: the relevant elasticity to estimate is $\sigma$ across task categories; the relevant macro outcome is the compute share of capital income; the relevant policy levers are compute-market levers, not labor-market levers. The numerical calibration in Section~\ref{sec:calib} suggests CAW is already binding on the high-volume substitutable margin and approaching binding on frontier-substitutable tasks at present compute prices.

\clearpage

\bibliographystyle{plainnat}
\bibliography{main}

@book{mankiw2020,
  author    = {Mankiw, N. Gregory},
  title     = {Principles of Economics},
  edition   = {9},
  publisher = {Cengage Learning},
  year      = {2020}
}

@article{katz1992,
  author  = {Katz, Lawrence F. and Murphy, Kevin M.},
  title   = {Changes in Relative Wages, 1963--1987: Supply and Demand Factors},
  journal = {Quarterly Journal of Economics},
  volume  = {107},
  number  = {1},
  pages   = {35--78},
  year    = {1992}
}

@article{autor2003,
  author  = {Autor, David H. and Levy, Frank and Murnane, Richard J.},
  title   = {The Skill Content of Recent Technological Change: An Empirical Exploration},
  journal = {Quarterly Journal of Economics},
  volume  = {118},
  number  = {4},
  pages   = {1279--1333},
  year    = {2003}
}

@article{autor2013,
  author  = {Autor, David H. and Dorn, David},
  title   = {The Growth of Low-Skill Service Jobs and the Polarization of the {US} Labor Market},
  journal = {American Economic Review},
  volume  = {103},
  number  = {5},
  pages   = {1553--1597},
  year    = {2013}
}

@article{autor2015,
  author  = {Autor, David H.},
  title   = {Why Are There Still So Many Jobs? The History and Future of Workplace Automation},
  journal = {Journal of Economic Perspectives},
  volume  = {29},
  number  = {3},
  pages   = {3--30},
  year    = {2015}
}

@book{goldinkatz2008,
  author    = {Goldin, Claudia and Katz, Lawrence F.},
  title     = {The Race Between Education and Technology},
  publisher = {Harvard University Press},
  year      = {2008}
}

@article{korv2000,
  author  = {Krusell, Per and Ohanian, Lee E. and R\'ios-Rull, Jos\'e-V\'ictor and Violante, Giovanni L.},
  title   = {Capital-Skill Complementarity and Inequality: A Macroeconomic Analysis},
  journal = {Econometrica},
  volume  = {68},
  number  = {5},
  pages   = {1029--1053},
  year    = {2000}
}

@article{acemoglu2018,
  author  = {Acemoglu, Daron and Restrepo, Pascual},
  title   = {The Race Between Man and Machine: Implications of Technology for Growth, Factor Shares, and Employment},
  journal = {American Economic Review},
  volume  = {108},
  number  = {6},
  pages   = {1488--1542},
  year    = {2018}
}

@article{acemoglu2019,
  author  = {Acemoglu, Daron and Restrepo, Pascual},
  title   = {Automation and New Tasks: How Technology Displaces and Reinstates Labor},
  journal = {Journal of Economic Perspectives},
  volume  = {33},
  number  = {2},
  pages   = {3--30},
  year    = {2019}
}

@article{acemoglu2020,
  author  = {Acemoglu, Daron and Restrepo, Pascual},
  title   = {Robots and Jobs: Evidence from {US} Labor Markets},
  journal = {Journal of Political Economy},
  volume  = {128},
  number  = {6},
  pages   = {2188--2244},
  year    = {2020}
}

@article{acemoglu2022,
  author  = {Acemoglu, Daron and Restrepo, Pascual},
  title   = {Tasks, Automation, and the Rise in {U.S.} Wage Inequality},
  journal = {Econometrica},
  volume  = {90},
  number  = {5},
  pages   = {1973--2016},
  year    = {2022}
}

@incollection{aghion2017,
  author    = {Aghion, Philippe and Jones, Benjamin F. and Jones, Charles I.},
  title     = {Artificial Intelligence and Economic Growth},
  booktitle = {The Economics of Artificial Intelligence: An Agenda},
  editor    = {Agrawal, Ajay and Gans, Joshua and Goldfarb, Avi},
  publisher = {University of Chicago Press},
  year      = {2019},
  pages     = {237--282}
}

@incollection{korinek2019,
  author    = {Korinek, Anton and Stiglitz, Joseph E.},
  title     = {Artificial Intelligence and Its Implications for Income Distribution and Unemployment},
  booktitle = {The Economics of Artificial Intelligence: An Agenda},
  editor    = {Agrawal, Ajay and Gans, Joshua and Goldfarb, Avi},
  publisher = {University of Chicago Press},
  year      = {2019},
  pages     = {349--390}
}

@article{korinek2023,
  author  = {Korinek, Anton},
  title   = {Language Models and Cognitive Automation for Economic Research},
  journal = {NBER Working Paper No.~30957},
  year    = {2023}
}

@article{trammellkorinek2023,
  author  = {Trammell, Philip and Korinek, Anton},
  title   = {Economic Growth under Transformative {AI}},
  journal = {Annual Review of Economics},
  volume  = {15},
  pages   = {567--593},
  year    = {2023}
}

@article{bresnahan1995,
  author  = {Bresnahan, Timothy F. and Trajtenberg, Manuel},
  title   = {General Purpose Technologies: ``{E}ngines of Growth''?},
  journal = {Journal of Econometrics},
  volume  = {65},
  number  = {1},
  pages   = {83--108},
  year    = {1995}
}

@article{goldfarb2023,
  author  = {Goldfarb, Avi and Taska, Bledi and Teodoridis, Florenta},
  title   = {Could Machine Learning Be a General Purpose Technology? A Comparison of Emerging Technologies Using Data from Online Job Postings},
  journal = {Research Policy},
  volume  = {52},
  number  = {1},
  pages   = {104653},
  year    = {2023}
}

@article{eloundou2023,
  author  = {Eloundou, Tyna and Manning, Sam and Mishkin, Pamela and Rock, Daniel},
  title   = {{GPTs} Are {GPTs}: An Early Look at the Labor Market Impact Potential of Large Language Models},
  journal = {arXiv preprint arXiv:2303.10130},
  year    = {2023}
}

@article{felten2023,
  author  = {Felten, Edward W. and Raj, Manav and Seamans, Robert},
  title   = {How Will Language Modelers Like {ChatGPT} Affect Occupations and Industries?},
  journal = {arXiv preprint arXiv:2303.01157},
  year    = {2023}
}

@article{brynjolfsson2023,
  author  = {Brynjolfsson, Erik and Li, Danielle and Raymond, Lindsey R.},
  title   = {Generative {AI} at Work},
  journal = {NBER Working Paper No.~31161},
  year    = {2023}
}

@article{noyzhang2023,
  author  = {Noy, Shakked and Zhang, Whitney},
  title   = {Experimental Evidence on the Productivity Effects of Generative Artificial Intelligence},
  journal = {Science},
  volume  = {381},
  number  = {6654},
  pages   = {187--192},
  year    = {2023}
}

@article{sevilla2022,
  author  = {Sevilla, Jaime and Heim, Lennart and Ho, Anson and Besiroglu, Tamay and Hobbhahn, Marius and Villalobos, Pablo},
  title   = {Compute Trends Across Three Eras of Machine Learning},
  journal = {2022 International Joint Conference on Neural Networks (IJCNN)},
  year    = {2022}
}

@article{cottier2024,
  author  = {Cottier, Ben and Rahman, Robi and Fattorini, Loredana and Maslej, Nestor and Owen, David},
  title   = {The Rising Costs of Training Frontier {AI} Models},
  journal = {arXiv preprint arXiv:2405.21015},
  year    = {2024}
}

@article{karabarbounis2014,
  author  = {Karabarbounis, Loukas and Neiman, Brent},
  title   = {The Global Decline of the Labor Share},
  journal = {Quarterly Journal of Economics},
  volume  = {129},
  number  = {1},
  pages   = {61--103},
  year    = {2014}
}

@book{piketty2014,
  author    = {Piketty, Thomas},
  title     = {Capital in the Twenty-First Century},
  publisher = {Harvard University Press},
  year      = {2014}
}

@article{autorvanreenen2020,
  author  = {Autor, David and Dorn, David and Katz, Lawrence F. and Patterson, Christina and Van Reenen, John},
  title   = {The Fall of the Labor Share and the Rise of Superstar Firms},
  journal = {Quarterly Journal of Economics},
  volume  = {135},
  number  = {2},
  pages   = {645--709},
  year    = {2020}
}

@book{susskind2020,
  author    = {Susskind, Daniel},
  title     = {A World Without Work: Technology, Automation, and How We Should Respond},
  publisher = {Metropolitan Books},
  year      = {2020}
}

\clearpage

\beginappendix

\section{Notation Summary}

For convenience we collect the symbols used throughout the paper.

\begin{itemize}[leftmargin=*]
\item $Y$: aggregate output of the representative competitive firm.
\item $K$, $L$: generic physical capital and labor in the \citet{mankiw2020} setup.
\item $K_o$: non-compute physical capital.
\item $K_c$: compute capital (GPUs, accelerators, data-center capacity, energy, model-weight IP rents).
\item $r_c$: competitive rental rate of compute capital, in \$/compute-unit-hour.
\item $L_H$: human cognitive labor, in labor-hours.
\item $L_A$: agent (cognitive-labor-equivalent) units, satisfying $L_A = \phi(K_c) \approx K_c / k$.
\item $k$: compute intensity, defined as units of compute capital required per effective agent-labor unit, in compute-units per agent-labor-hour.
\item $\phi$: capital-to-labor conversion technology embedding model architecture and inference stack.
\item $\lambda$: relative human-to-agent productivity on the substitutable task set, where one human-labor unit produces the same effective output as $\lambda$ agent-labor units. Values $\lambda > 1$ correspond to humans being more productive per unit, while $\lambda < 1$ corresponds to agents dominating.
\item $W_H$: human cognitive wage, in \$/labor-hour.
\item $W_A^{\text{eff}} \equiv k r_c$: effective agent unit wage, in \$/agent-labor-hour.
\item $\sigma$: elasticity of substitution between human and agent cognitive labor in the CES aggregator.
\item $T_S$, $T_C$: substitutable and complementary cognitive task sets.
\end{itemize}

\section{Detailed Derivation of the CAW Bound}

We expand the proof sketch of Proposition~\ref{prop:caw}. Under Assumption~\ref{assn:perfect}, the effective cognitive-labor input on the substitutable task set is $L_{\text{eff}} = L_H + \lambda^{-1} L_A$. A profit-maximizing firm chooses $(L_H, L_A)$ to minimize the cost of producing one unit of $L_{\text{eff}}$:
\[
\min_{L_H,\, L_A \geq 0} \; W_H L_H + r_c \, k\, L_A \quad \text{s.t.} \quad L_H + \lambda^{-1} L_A \geq 1.
\]
Since the constraint is linear and the objective is linear, the solution is a corner whenever the per-unit cost $W_H$ of $L_H$ and the per-unit cost $\lambda k r_c$ of effective labor delivered through $L_A$ are unequal. If $W_H < \lambda k r_c$, the firm substitutes fully toward $L_H$, so that $L_A^d = 0$ on $T_S$. If instead $W_H > \lambda k r_c$, the firm substitutes fully toward $L_A$, so that $L_H^d = 0$ on $T_S$; any positive supply of $L_H$ at this wage on this task is therefore unemployed, and the wage cannot be sustained in any equilibrium with $L_H > 0$ on $T_S$. Interior coexistence in turn requires $W_H = \lambda k r_c$. Across all cases consistent with positive employment, $W_H \leq \lambda k r_c$ on the substitutable margin, with equality whenever both factors are simultaneously employed. When $L_A = 0$ everywhere, for example because compute is unavailable or prohibitively expensive, the standard labor-market wage prevails and the bound is slack.

\section{CES Algebra}

Cost minimization of the CES aggregator \eqref{eq:ces} subject to producing one unit of $L_{\text{eff}}$ yields the conditional factor demands
\[
L_H = \frac{1}{A}\Bigl(\frac{\alpha}{W_H}\Bigr)^{\sigma} \Lambda, \qquad
L_A = \frac{1}{A}\Bigl(\frac{\beta}{W_A^{\text{eff}}}\Bigr)^{\sigma} \Lambda,
\]
where $\Lambda$ is the dual CES price index (a function of $W_H$ and $W_A^{\text{eff}}$ alone). Taking the ratio gives \eqref{eq:relwage}.

The two limits are immediate. As $\sigma \to \infty$ (perfect substitutes), the relative wage \eqref{eq:relwage} forces $W_H \to (\alpha/\beta)\, W_A^{\text{eff}}$, which under the normalization $\alpha/\beta = \lambda$ recovers the perfect-substitute CAW bound. As $\sigma \to 0$ (Leontief), the conditional demand for $L_H$ is fixed by the technology in proportion to $L_A$, so $L_H$ and $L_A$ are used together in fixed ratio; the binding factor is then whichever is shorter in supply. If human-labor supply is the binding constraint at the prevailing demand for $L_{\text{eff}}$, then $W_H$ is determined by $L_H^s$ and the comparative-static channel from $W_A^{\text{eff}}$ to $W_H$ vanishes; if compute is the binding constraint, then $W_H$ tracks $W_A^{\text{eff}}$ scaled by the technology proportion.

For Proposition~\ref{prop:compression}, holding $L_H^s$ and the demand for $L_{\text{eff}}$ fixed, totally differentiating \eqref{eq:relwage} with respect to $W_A^{\text{eff}}$ yields
\[
\frac{d \log W_H}{d \log W_A^{\text{eff}}} = 1 - \frac{1}{\sigma}\cdot\frac{d \log(L_H/L_A)}{d \log W_A^{\text{eff}}}.
\]
Under perfectly elastic $L_H^s$, $W_H$ tracks $W_A^{\text{eff}}$ one-for-one and the second term vanishes. Under inelastic $L_H^s$, the cross-derivative term partially offsets the direct effect, with the magnitude controlled by $1/\sigma$.

\end{document}